\documentclass[11pt]{article}

\usepackage[margin=1in]{geometry}
\usepackage{amsmath,amssymb}
\usepackage{amsthm}
\usepackage{booktabs}

\newtheorem{proposition}{Proposition}
\theoremstyle{remark}
\newtheorem{remark}{Remark}
\usepackage{graphicx}
\usepackage{adjustbox}
\usepackage{hyperref}
\usepackage[numbers,sort&compress]{natbib}

\title{Distributional Split Criteria for Random Forests:\\
Extensions, Shrinkage, and the Robustness of Mean Splitting}
\author{Silas Koemen}
\date{\today}

\newcommand{\mmd}{\mathrm{MMD}}

\begin{document}
\maketitle

\begin{abstract}
Distributional random forests replace mean-based CART splitting with criteria
that compare the full conditional response distribution in candidate children.
We implement and systematically study a family of such criteria inside a single
honest-forest implementation: isotropic random-Fourier-feature maximum mean
discrepancy (MMD), an anisotropic diagonal-bandwidth variant, an adaptive
per-split frequency-selection variant, and a non-kernel sliced-Wasserstein
criterion, together with post-hoc kernel-mean shrinkage of the forest weights.
Using paired-seed comparisons across synthetic quantile mechanisms, real
univariate benchmarks, a California-housing subsample curve, and multivariate
synthetic and real responses, we characterize where each extension pays.  Three
findings recur.  First, among distributional criteria ordinary isotropic MMD is
already close to best in class: the anisotropic, adaptive-frequency, and
sliced-Wasserstein extensions, and post-hoc shrinkage, do not systematically
improve on it.  Second, on scalar tabular regression mean-based CART splitting
remains the robust default and wins many cells.  Third, multivariate responses
are the regime where distributional splitting clearly earns its keep, most
sharply on a pure-dependence copula where the energy score separates the
criteria even though marginal CRPS does not.  The evidence supports a simple
allocation story: distributional splitting helps only when non-location
structure is both present and estimable; otherwise it dilutes split-selection
power away from the mean.  All criteria, the honest forest, and the
paired-comparison harness are implemented in the open-source \texttt{drforest}
library, whose Rust-backed split search makes broad criterion sweeps inexpensive.
\end{abstract}

\section{Question}

Random forests for distributional prediction estimate not just a conditional
mean, but a conditional law.  Quantile regression forests do this with standard
CART tree structure and distributional leaf weights \citep{meinshausen2006}.
Distributional random forests instead propose splitting rules that compare the
full response distributions in candidate children \citep{cevid2022}.  This is a
natural generalization, especially for multivariate responses where a scalar
ordering or separate marginal quantile fits are unsatisfactory.

This note studies that design space directly.  Inside one honest-forest
implementation we add four distributional split criteria---isotropic MMD, a
diagonal-bandwidth anisotropic variant, an adaptive per-split
frequency-selection variant, and sliced Wasserstein---and a family of post-hoc
shrinkage variants, and compare them by paired seed.  Two questions organize the
study: \emph{among distributional criteria, does any extension beat the plainest
MMD?} and \emph{against mean splitting, when is distributional splitting worth
its variance?}  The answers are, respectively, ``not systematically'' and ``for
multivariate responses and a few structured scalar problems,'' with CART the
robust default elsewhere.  We test the multivariate regime explicitly, since it
is what motivates distributional forests in the first place.

\section{Split Criteria as Signal Allocation}

For a node with observations $(X_i,Y_i)$, a candidate split $A \mid A^c$ is
chosen by maximizing an empirical separation score between the child responses.
CART uses the between-group mean separation
\[
  \frac{n_L n_R}{n^2}\|\bar Y_L-\bar Y_R\|^2,
\]
which is the identity-kernel mean-embedding score.  Random Fourier feature MMD
splitting replaces $Y$ by a complex feature map $\psi(Y)$ and scores
\[
  \frac{n_L n_R}{n^2}
  \left\| \frac{1}{n_L}\sum_{i\in L}\psi(Y_i)
        - \frac{1}{n_R}\sum_{i\in R}\psi(Y_i) \right\|^2 .
\]
With enough features and enough data this score can detect changes beyond the
mean.  The two scores share an exact finite-sample structure: each is a biased
squared mean-embedding statistic whose expectation is the population separation
plus an $O(n^{-1})$ noise floor set by the embedded-response covariance
(Appendix~\ref{app:finite}).  The embedding, not the form of the statistic, is
all that changes between CART and MMD.

This suggests the working characterization tested below.  CART concentrates on
the dominant location signal.  MMD spreads sensitivity across location and
higher-order distributional structure.  At finite node sizes the usefulness of
distributional splitting is governed by the kernel-mean contrast relative to the
covariance of the embedded response, together with random-feature and
split-search noise; universal kernel sensitivity does not imply high split-time
power for every alternative.  MMD should therefore win only when the additional
structure is strong enough, or the sample size large enough, for its split-time
signal-to-noise ratio to exceed the cost of estimating it (Appendix~\ref{app:finite}).

We compare five criteria: CART; isotropic Gaussian RFF MMD (\texttt{mmd\_rff});
diagonal-bandwidth MMD (\texttt{anisotropic\_mmd}); an overcomplete RFF pool with
per-split top-$k$ frequency selection (\texttt{adaptive\_mmd}); and a non-kernel
sliced-Wasserstein criterion.\footnote{Diagonal-bandwidth MMD is reported only
for multivariate responses.  On a one-dimensional response it is not a distinct
estimator: the coordinatewise median bandwidth coincides with the Euclidean
median, and the length-one frequency draw is the same realization as the
isotropic scalar-scale draw, so \texttt{anisotropic\_mmd} reproduces
\texttt{mmd\_rff} bit-for-bit and is omitted from the scalar tables.}

\section{Evidence}

The experiments were run as paired-seed comparisons: the train/test split,
forest seeds, tree hyperparameters, honesty setting, and cutpoint grid were held
fixed while varying only the split criterion.  Metrics are RMSE for the
conditional mean and CRPS / energy score for distributional prediction.  In
every table a negative difference means the row improves over the reference, and
the parenthetical is the seed-level win rate.

\paragraph{Synthetic diagnostics.}
The first set of experiments swept the univariate synthetic quantile mechanisms
used in the DRF literature (Table~\ref{tab:synthetic}).  When the conditional law
carries deliberately strong non-location structure, distributional splitting
pays: on \texttt{paper\_quantile\_2} and \texttt{paper\_quantile\_3} with honesty
$0.5$, every distributional criterion improves CRPS and energy over CART with a
100\% win rate while leaving RMSE essentially unchanged.  On the more
location-driven \texttt{paper\_quantile\_1} the comparison is a wash.  Honesty is
a visible lever in its own right: the gap between criteria at honesty $0$ largely
collapses once sample splitting is turned on.

\begin{table}[h]
  \centering
  \caption{Paired differences (criterion $-$ \texttt{cart}). Negative means the row improves over the reference; values are mean $\pm$ paired SE with seed-level win rate.}
  \label{tab:synthetic}
  \begin{adjustbox}{max width=\textwidth}
  \begin{tabular}{lccccccc}
    \toprule
    dataset & $n_{\mathrm{tr}}$ & honesty & criterion & pairs & $\Delta$RMSE & $\Delta$CRPS & $\Delta$energy \\
    \midrule
    paper\_quantile\_1 & 1500 & 0 & adaptive\_mmd & 10 & -0.0069 $\pm$ 0.0012 (90\%) & -0.0032 $\pm$ 0.0006 (100\%) & -0.0032 $\pm$ 0.0006 (100\%) \\
    paper\_quantile\_1 & 1500 & 0 & mmd\_rff & 10 & -0.0026 $\pm$ 0.0011 (80\%) & -0.0006 $\pm$ 0.0006 (60\%) & -0.0006 $\pm$ 0.0006 (60\%) \\
    paper\_quantile\_1 & 1500 & 0 & sliced\_wasserstein & 10 & -0.0011 $\pm$ 0.0008 (60\%) & -0.0001 $\pm$ 0.0005 (40\%) & -0.0001 $\pm$ 0.0005 (40\%) \\
    paper\_quantile\_1 & 1500 & 0.5 & adaptive\_mmd & 10 & +0.0003 $\pm$ 0.0006 (20\%) & +0.0001 $\pm$ 0.0003 (30\%) & +0.0001 $\pm$ 0.0003 (30\%) \\
    paper\_quantile\_1 & 1500 & 0.5 & mmd\_rff & 10 & +0.0002 $\pm$ 0.0003 (40\%) & +0.0001 $\pm$ 0.0002 (30\%) & +0.0001 $\pm$ 0.0002 (30\%) \\
    paper\_quantile\_1 & 1500 & 0.5 & sliced\_wasserstein & 10 & -0.0000 $\pm$ 0.0004 (50\%) & -0.0001 $\pm$ 0.0002 (50\%) & -0.0001 $\pm$ 0.0002 (50\%) \\
    paper\_quantile\_2 & 1500 & 0 & adaptive\_mmd & 10 & -0.0082 $\pm$ 0.0018 (100\%) & -0.0089 $\pm$ 0.0011 (100\%) & -0.0089 $\pm$ 0.0011 (100\%) \\
    paper\_quantile\_2 & 1500 & 0 & mmd\_rff & 10 & -0.0026 $\pm$ 0.0024 (70\%) & -0.0060 $\pm$ 0.0014 (90\%) & -0.0060 $\pm$ 0.0014 (90\%) \\
    paper\_quantile\_2 & 1500 & 0 & sliced\_wasserstein & 10 & +0.0010 $\pm$ 0.0016 (30\%) & -0.0054 $\pm$ 0.0010 (100\%) & -0.0054 $\pm$ 0.0010 (100\%) \\
    paper\_quantile\_2 & 1500 & 0.5 & adaptive\_mmd & 10 & +0.0001 $\pm$ 0.0005 (70\%) & -0.0105 $\pm$ 0.0011 (100\%) & -0.0105 $\pm$ 0.0011 (100\%) \\
    paper\_quantile\_2 & 1500 & 0.5 & mmd\_rff & 10 & +0.0003 $\pm$ 0.0004 (40\%) & -0.0105 $\pm$ 0.0009 (100\%) & -0.0105 $\pm$ 0.0009 (100\%) \\
    paper\_quantile\_2 & 1500 & 0.5 & sliced\_wasserstein & 10 & +0.0007 $\pm$ 0.0006 (40\%) & -0.0105 $\pm$ 0.0011 (100\%) & -0.0105 $\pm$ 0.0011 (100\%) \\
    paper\_quantile\_3 & 1500 & 0 & adaptive\_mmd & 10 & -0.0057 $\pm$ 0.0017 (90\%) & -0.0043 $\pm$ 0.0015 (80\%) & -0.0043 $\pm$ 0.0015 (80\%) \\
    paper\_quantile\_3 & 1500 & 0 & mmd\_rff & 10 & -0.0018 $\pm$ 0.0015 (80\%) & -0.0023 $\pm$ 0.0012 (80\%) & -0.0023 $\pm$ 0.0012 (80\%) \\
    paper\_quantile\_3 & 1500 & 0 & sliced\_wasserstein & 10 & +0.0015 $\pm$ 0.0015 (40\%) & -0.0021 $\pm$ 0.0009 (70\%) & -0.0021 $\pm$ 0.0009 (70\%) \\
    paper\_quantile\_3 & 1500 & 0.5 & adaptive\_mmd & 10 & +0.0007 $\pm$ 0.0005 (30\%) & -0.0034 $\pm$ 0.0008 (100\%) & -0.0034 $\pm$ 0.0008 (100\%) \\
    paper\_quantile\_3 & 1500 & 0.5 & mmd\_rff & 10 & +0.0003 $\pm$ 0.0004 (50\%) & -0.0033 $\pm$ 0.0006 (100\%) & -0.0033 $\pm$ 0.0006 (100\%) \\
    paper\_quantile\_3 & 1500 & 0.5 & sliced\_wasserstein & 10 & +0.0004 $\pm$ 0.0003 (30\%) & -0.0030 $\pm$ 0.0004 (100\%) & -0.0030 $\pm$ 0.0004 (100\%) \\
    \bottomrule
  \end{tabular}
  \end{adjustbox}
\end{table}

\paragraph{Real univariate data.}
On real univariate benchmarks CART-style splitting is again the stronger default
(Table~\ref{tab:real}).  CART wins \texttt{wine\_quality\_white} and
\texttt{kin8nm} against the MMD family, and \texttt{diabetes} ($n=300$) is too
noisy to separate the criteria.  The exception worth isolating is California
housing, where the MMD family beats CART consistently (isotropic MMD improves
CRPS by $\approx 0.003$ at a 100\% win rate, and adaptive MMD roughly doubles
that).  That case confounds two explanations: the dataset is large, and housing
prices may contain estimable heteroskedastic structure.  A subsample experiment
over $n\in\{2000,4000,8000,16000\}$ separates these readings
(Figure~\ref{fig:california}).  The paired MMD$-$CART advantage is already
present at small $n$ and is sustained as $n$ grows, which favors the
strong-distributional-structure reading over a pure estimability threshold.

\begin{table}[h]
  \centering
  \caption{Paired differences (criterion $-$ \texttt{cart}). Negative means the row improves over the reference; values are mean $\pm$ paired SE with seed-level win rate.}
  \label{tab:real}
  \begin{adjustbox}{max width=\textwidth}
  \begin{tabular}{lccccccc}
    \toprule
    dataset & $n_{\mathrm{tr}}$ & honesty & criterion & pairs & $\Delta$RMSE & $\Delta$CRPS & $\Delta$energy \\
    \midrule
    california\_housing & 3000 & 0 & adaptive\_mmd & 10 & -0.0051 $\pm$ 0.0019 (80\%) & -0.0053 $\pm$ 0.0008 (100\%) & -0.0053 $\pm$ 0.0008 (100\%) \\
    california\_housing & 3000 & 0 & mmd\_rff & 10 & -0.0032 $\pm$ 0.0006 (90\%) & -0.0026 $\pm$ 0.0003 (100\%) & -0.0026 $\pm$ 0.0003 (100\%) \\
    california\_housing & 3000 & 0 & sliced\_wasserstein & 10 & -0.0005 $\pm$ 0.0008 (70\%) & -0.0008 $\pm$ 0.0004 (80\%) & -0.0008 $\pm$ 0.0004 (80\%) \\
    california\_housing & 3000 & 0.5 & adaptive\_mmd & 10 & -0.0075 $\pm$ 0.0014 (90\%) & -0.0055 $\pm$ 0.0007 (100\%) & -0.0055 $\pm$ 0.0007 (100\%) \\
    california\_housing & 3000 & 0.5 & mmd\_rff & 10 & -0.0045 $\pm$ 0.0006 (100\%) & -0.0029 $\pm$ 0.0003 (100\%) & -0.0029 $\pm$ 0.0003 (100\%) \\
    california\_housing & 3000 & 0.5 & sliced\_wasserstein & 10 & -0.0017 $\pm$ 0.0006 (80\%) & -0.0010 $\pm$ 0.0003 (80\%) & -0.0010 $\pm$ 0.0003 (80\%) \\
    diabetes & 300 & 0 & adaptive\_mmd & 10 & +0.1814 $\pm$ 0.1581 (50\%) & +0.2474 $\pm$ 0.0867 (20\%) & +0.2474 $\pm$ 0.0867 (20\%) \\
    diabetes & 300 & 0 & mmd\_rff & 10 & +0.0565 $\pm$ 0.0947 (50\%) & +0.1077 $\pm$ 0.0572 (30\%) & +0.1077 $\pm$ 0.0572 (30\%) \\
    diabetes & 300 & 0 & sliced\_wasserstein & 10 & +0.0831 $\pm$ 0.1027 (50\%) & +0.0811 $\pm$ 0.0684 (30\%) & +0.0811 $\pm$ 0.0684 (30\%) \\
    diabetes & 300 & 0.5 & adaptive\_mmd & 10 & +0.1425 $\pm$ 0.0994 (30\%) & +0.1065 $\pm$ 0.0549 (30\%) & +0.1065 $\pm$ 0.0549 (30\%) \\
    diabetes & 300 & 0.5 & mmd\_rff & 10 & +0.0222 $\pm$ 0.0595 (50\%) & +0.0176 $\pm$ 0.0363 (50\%) & +0.0176 $\pm$ 0.0363 (50\%) \\
    diabetes & 300 & 0.5 & sliced\_wasserstein & 10 & +0.0401 $\pm$ 0.0655 (40\%) & +0.0220 $\pm$ 0.0368 (50\%) & +0.0220 $\pm$ 0.0368 (50\%) \\
    kin8nm & 3000 & 0 & adaptive\_mmd & 10 & +0.0069 $\pm$ 0.0003 (0\%) & +0.0031 $\pm$ 0.0001 (0\%) & +0.0031 $\pm$ 0.0001 (0\%) \\
    kin8nm & 3000 & 0 & mmd\_rff & 10 & +0.0014 $\pm$ 0.0002 (10\%) & +0.0005 $\pm$ 0.0001 (10\%) & +0.0005 $\pm$ 0.0001 (10\%) \\
    kin8nm & 3000 & 0 & sliced\_wasserstein & 10 & -0.0016 $\pm$ 0.0002 (100\%) & -0.0008 $\pm$ 0.0001 (100\%) & -0.0008 $\pm$ 0.0001 (100\%) \\
    kin8nm & 3000 & 0.5 & adaptive\_mmd & 10 & +0.0035 $\pm$ 0.0003 (0\%) & +0.0014 $\pm$ 0.0001 (0\%) & +0.0014 $\pm$ 0.0001 (0\%) \\
    kin8nm & 3000 & 0.5 & mmd\_rff & 10 & +0.0007 $\pm$ 0.0001 (0\%) & +0.0001 $\pm$ 0.0001 (40\%) & +0.0001 $\pm$ 0.0001 (40\%) \\
    kin8nm & 3000 & 0.5 & sliced\_wasserstein & 10 & -0.0006 $\pm$ 0.0001 (100\%) & -0.0004 $\pm$ 0.0000 (100\%) & -0.0004 $\pm$ 0.0000 (100\%) \\
    wine\_quality\_white & 3000 & 0 & adaptive\_mmd & 10 & +0.0191 $\pm$ 0.0015 (0\%) & +0.0075 $\pm$ 0.0009 (0\%) & +0.0075 $\pm$ 0.0009 (0\%) \\
    wine\_quality\_white & 3000 & 0 & mmd\_rff & 10 & +0.0068 $\pm$ 0.0009 (0\%) & +0.0021 $\pm$ 0.0005 (10\%) & +0.0021 $\pm$ 0.0005 (10\%) \\
    wine\_quality\_white & 3000 & 0 & sliced\_wasserstein & 10 & +0.0072 $\pm$ 0.0010 (0\%) & +0.0024 $\pm$ 0.0005 (10\%) & +0.0024 $\pm$ 0.0005 (10\%) \\
    wine\_quality\_white & 3000 & 0.5 & adaptive\_mmd & 10 & +0.0073 $\pm$ 0.0009 (0\%) & +0.0035 $\pm$ 0.0005 (0\%) & +0.0035 $\pm$ 0.0005 (0\%) \\
    wine\_quality\_white & 3000 & 0.5 & mmd\_rff & 10 & +0.0026 $\pm$ 0.0005 (0\%) & +0.0011 $\pm$ 0.0002 (10\%) & +0.0011 $\pm$ 0.0002 (10\%) \\
    wine\_quality\_white & 3000 & 0.5 & sliced\_wasserstein & 10 & +0.0043 $\pm$ 0.0006 (0\%) & +0.0023 $\pm$ 0.0003 (0\%) & +0.0023 $\pm$ 0.0003 (0\%) \\
    \bottomrule
  \end{tabular}
  \end{adjustbox}
\end{table}

\begin{figure}[h]
  \centering
  \includegraphics[width=0.62\textwidth]{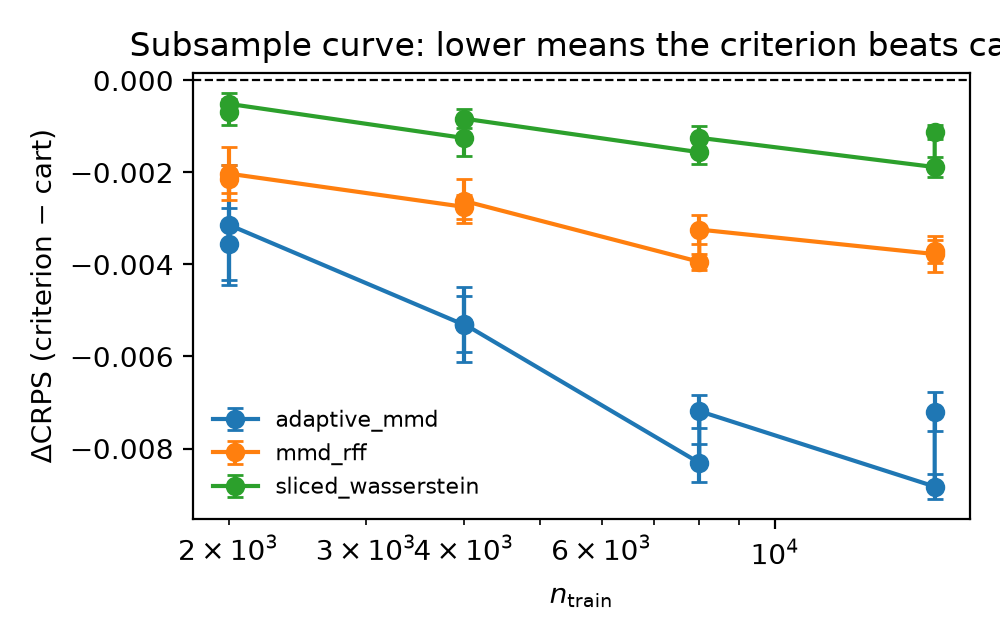}
  \caption{California housing subsample curve: paired CRPS difference (criterion
  $-$ CART) versus $n_{\mathrm{train}}$.  Points below zero mean the criterion
  beats CART.  The MMD family is below CART across the whole range rather than
  only at large $n$.}
  \label{fig:california}
\end{figure}

\paragraph{Do the MMD refinements help?}
Table~\ref{tab:modification} references plain isotropic MMD rather than CART and
asks whether any refinement is a better distributional baseline.  None is.
Adaptive frequency selection helps on California but hurts on \texttt{kin8nm} and
\texttt{wine}; sliced Wasserstein helps on \texttt{kin8nm} but loses elsewhere;
on the multivariate cells the refinements scatter around zero relative to
\texttt{mmd\_rff} with no consistent sign.  This is the evidence that ordinary
RFF-MMD is already the right distributional default: the extra machinery does not
move the frontier.  Adaptive-frequency MMD in particular was included as a rescue
attempt---if uniform RFF averaging were hiding useful signal, per-split selection
should expose it---and it does not change the conclusion, which points away from
a simple tuning failure and toward the finite-sample allocation explanation.

\begin{table}[h]
  \centering
  \caption{Paired differences (criterion $-$ \texttt{mmd\_rff}). Negative means the row improves over the reference; values are mean $\pm$ paired SE with seed-level win rate.}
  \label{tab:modification}
  \begin{adjustbox}{max width=\textwidth}
  \begin{tabular}{lccccccc}
    \toprule
    dataset & $n_{\mathrm{tr}}$ & honesty & criterion & pairs & $\Delta$RMSE & $\Delta$CRPS & $\Delta$energy \\
    \midrule
    california\_housing & 3000 & 0 & adaptive\_mmd & 10 & -0.0020 $\pm$ 0.0015 (60\%) & -0.0028 $\pm$ 0.0007 (100\%) & -0.0028 $\pm$ 0.0007 (100\%) \\
    california\_housing & 3000 & 0 & cart & 10 & +0.0032 $\pm$ 0.0006 (10\%) & +0.0026 $\pm$ 0.0003 (0\%) & +0.0026 $\pm$ 0.0003 (0\%) \\
    california\_housing & 3000 & 0 & sliced\_wasserstein & 10 & +0.0026 $\pm$ 0.0011 (30\%) & +0.0017 $\pm$ 0.0005 (10\%) & +0.0017 $\pm$ 0.0005 (10\%) \\
    california\_housing & 3000 & 0.5 & adaptive\_mmd & 10 & -0.0031 $\pm$ 0.0011 (80\%) & -0.0026 $\pm$ 0.0006 (90\%) & -0.0026 $\pm$ 0.0006 (90\%) \\
    california\_housing & 3000 & 0.5 & cart & 10 & +0.0045 $\pm$ 0.0006 (0\%) & +0.0029 $\pm$ 0.0003 (0\%) & +0.0029 $\pm$ 0.0003 (0\%) \\
    california\_housing & 3000 & 0.5 & sliced\_wasserstein & 10 & +0.0027 $\pm$ 0.0010 (30\%) & +0.0019 $\pm$ 0.0005 (20\%) & +0.0019 $\pm$ 0.0005 (20\%) \\
    diabetes & 300 & 0 & adaptive\_mmd & 10 & +0.1249 $\pm$ 0.1012 (40\%) & +0.1397 $\pm$ 0.0634 (30\%) & +0.1397 $\pm$ 0.0634 (30\%) \\
    diabetes & 300 & 0 & cart & 10 & -0.0565 $\pm$ 0.0947 (50\%) & -0.1077 $\pm$ 0.0572 (70\%) & -0.1077 $\pm$ 0.0572 (70\%) \\
    diabetes & 300 & 0 & sliced\_wasserstein & 10 & +0.0266 $\pm$ 0.1016 (70\%) & -0.0266 $\pm$ 0.0641 (70\%) & -0.0266 $\pm$ 0.0641 (70\%) \\
    diabetes & 300 & 0.5 & adaptive\_mmd & 10 & +0.1203 $\pm$ 0.0578 (30\%) & +0.0890 $\pm$ 0.0323 (40\%) & +0.0890 $\pm$ 0.0323 (40\%) \\
    diabetes & 300 & 0.5 & cart & 10 & -0.0222 $\pm$ 0.0595 (50\%) & -0.0176 $\pm$ 0.0363 (50\%) & -0.0176 $\pm$ 0.0363 (50\%) \\
    diabetes & 300 & 0.5 & sliced\_wasserstein & 10 & +0.0178 $\pm$ 0.0316 (40\%) & +0.0045 $\pm$ 0.0233 (50\%) & +0.0045 $\pm$ 0.0233 (50\%) \\
    enb & 576 & 0 & adaptive\_mmd & 5 & -0.0034 $\pm$ 0.0018 (80\%) & -0.0017 $\pm$ 0.0010 (80\%) & -0.0023 $\pm$ 0.0014 (80\%) \\
    enb & 576 & 0 & anisotropic\_mmd & 5 & -0.0021 $\pm$ 0.0025 (60\%) & -0.0010 $\pm$ 0.0010 (40\%) & -0.0013 $\pm$ 0.0015 (60\%) \\
    enb & 576 & 0 & cart & 5 & +0.0027 $\pm$ 0.0027 (40\%) & +0.0017 $\pm$ 0.0011 (20\%) & +0.0024 $\pm$ 0.0017 (20\%) \\
    enb & 576 & 0 & sliced\_wasserstein & 5 & +0.0090 $\pm$ 0.0032 (0\%) & +0.0029 $\pm$ 0.0013 (20\%) & +0.0040 $\pm$ 0.0018 (20\%) \\
    enb & 576 & 0.5 & adaptive\_mmd & 5 & -0.0064 $\pm$ 0.0044 (60\%) & -0.0025 $\pm$ 0.0016 (80\%) & -0.0036 $\pm$ 0.0024 (100\%) \\
    enb & 576 & 0.5 & anisotropic\_mmd & 5 & -0.0007 $\pm$ 0.0016 (60\%) & -0.0001 $\pm$ 0.0004 (40\%) & -0.0004 $\pm$ 0.0006 (60\%) \\
    enb & 576 & 0.5 & cart & 5 & +0.0071 $\pm$ 0.0037 (20\%) & +0.0027 $\pm$ 0.0013 (0\%) & +0.0036 $\pm$ 0.0017 (0\%) \\
    enb & 576 & 0.5 & sliced\_wasserstein & 5 & +0.0089 $\pm$ 0.0031 (0\%) & +0.0047 $\pm$ 0.0010 (0\%) & +0.0058 $\pm$ 0.0013 (0\%) \\
    kin8nm & 3000 & 0 & adaptive\_mmd & 10 & +0.0056 $\pm$ 0.0002 (0\%) & +0.0026 $\pm$ 0.0001 (0\%) & +0.0026 $\pm$ 0.0001 (0\%) \\
    kin8nm & 3000 & 0 & cart & 10 & -0.0014 $\pm$ 0.0002 (90\%) & -0.0005 $\pm$ 0.0001 (90\%) & -0.0005 $\pm$ 0.0001 (90\%) \\
    kin8nm & 3000 & 0 & sliced\_wasserstein & 10 & -0.0030 $\pm$ 0.0003 (100\%) & -0.0013 $\pm$ 0.0001 (100\%) & -0.0013 $\pm$ 0.0001 (100\%) \\
    kin8nm & 3000 & 0.5 & adaptive\_mmd & 10 & +0.0028 $\pm$ 0.0002 (0\%) & +0.0013 $\pm$ 0.0001 (0\%) & +0.0013 $\pm$ 0.0001 (0\%) \\
    kin8nm & 3000 & 0.5 & cart & 10 & -0.0007 $\pm$ 0.0001 (100\%) & -0.0001 $\pm$ 0.0001 (60\%) & -0.0001 $\pm$ 0.0001 (60\%) \\
    kin8nm & 3000 & 0.5 & sliced\_wasserstein & 10 & -0.0013 $\pm$ 0.0001 (100\%) & -0.0005 $\pm$ 0.0001 (100\%) & -0.0005 $\pm$ 0.0001 (100\%) \\
    paper\_copula & 3750 & 0 & adaptive\_mmd & 5 & -0.0007 $\pm$ 0.0003 (80\%) & -0.0003 $\pm$ 0.0002 (60\%) & -0.0010 $\pm$ 0.0005 (100\%) \\
    paper\_copula & 3750 & 0 & anisotropic\_mmd & 5 & -0.0008 $\pm$ 0.0006 (60\%) & +0.0002 $\pm$ 0.0004 (40\%) & -0.0005 $\pm$ 0.0008 (60\%) \\
    paper\_copula & 3750 & 0 & cart & 5 & +0.0009 $\pm$ 0.0005 (40\%) & +0.0004 $\pm$ 0.0003 (40\%) & +0.0045 $\pm$ 0.0007 (0\%) \\
    paper\_copula & 3750 & 0 & sliced\_wasserstein & 5 & +0.0003 $\pm$ 0.0007 (40\%) & +0.0002 $\pm$ 0.0004 (60\%) & -0.0002 $\pm$ 0.0009 (60\%) \\
    paper\_copula & 3750 & 0.5 & adaptive\_mmd & 5 & -0.0001 $\pm$ 0.0002 (40\%) & -0.0000 $\pm$ 0.0001 (40\%) & -0.0002 $\pm$ 0.0003 (60\%) \\
    paper\_copula & 3750 & 0.5 & anisotropic\_mmd & 5 & +0.0003 $\pm$ 0.0003 (20\%) & +0.0003 $\pm$ 0.0001 (20\%) & -0.0005 $\pm$ 0.0003 (80\%) \\
    paper\_copula & 3750 & 0.5 & cart & 5 & -0.0002 $\pm$ 0.0003 (60\%) & -0.0001 $\pm$ 0.0001 (60\%) & +0.0038 $\pm$ 0.0003 (0\%) \\
    paper\_copula & 3750 & 0.5 & sliced\_wasserstein & 5 & -0.0002 $\pm$ 0.0003 (60\%) & -0.0000 $\pm$ 0.0002 (60\%) & -0.0008 $\pm$ 0.0004 (80\%) \\
    paper\_heterogeneous\_regression & 3750 & 0 & adaptive\_mmd & 5 & -0.0008 $\pm$ 0.0010 (80\%) & +0.0001 $\pm$ 0.0008 (60\%) & +0.0001 $\pm$ 0.0011 (60\%) \\
    paper\_heterogeneous\_regression & 3750 & 0 & anisotropic\_mmd & 5 & -0.0005 $\pm$ 0.0007 (60\%) & +0.0005 $\pm$ 0.0005 (40\%) & +0.0008 $\pm$ 0.0007 (20\%) \\
    paper\_heterogeneous\_regression & 3750 & 0 & cart & 5 & -0.0012 $\pm$ 0.0018 (60\%) & +0.0006 $\pm$ 0.0008 (20\%) & +0.0010 $\pm$ 0.0012 (20\%) \\
    paper\_heterogeneous\_regression & 3750 & 0 & sliced\_wasserstein & 5 & -0.0012 $\pm$ 0.0007 (80\%) & -0.0005 $\pm$ 0.0003 (80\%) & -0.0008 $\pm$ 0.0005 (80\%) \\
    paper\_heterogeneous\_regression & 3750 & 0.5 & adaptive\_mmd & 5 & +0.0020 $\pm$ 0.0012 (20\%) & +0.0004 $\pm$ 0.0006 (20\%) & +0.0007 $\pm$ 0.0010 (20\%) \\
    paper\_heterogeneous\_regression & 3750 & 0.5 & anisotropic\_mmd & 5 & +0.0007 $\pm$ 0.0008 (40\%) & +0.0000 $\pm$ 0.0003 (40\%) & +0.0002 $\pm$ 0.0006 (40\%) \\
    paper\_heterogeneous\_regression & 3750 & 0.5 & cart & 5 & +0.0023 $\pm$ 0.0010 (0\%) & +0.0042 $\pm$ 0.0004 (0\%) & +0.0069 $\pm$ 0.0008 (0\%) \\
    paper\_heterogeneous\_regression & 3750 & 0.5 & sliced\_wasserstein & 5 & +0.0013 $\pm$ 0.0005 (20\%) & +0.0012 $\pm$ 0.0002 (0\%) & +0.0019 $\pm$ 0.0003 (0\%) \\
    wine\_quality\_white & 3000 & 0 & adaptive\_mmd & 10 & +0.0123 $\pm$ 0.0006 (0\%) & +0.0054 $\pm$ 0.0004 (0\%) & +0.0054 $\pm$ 0.0004 (0\%) \\
    wine\_quality\_white & 3000 & 0 & cart & 10 & -0.0068 $\pm$ 0.0009 (100\%) & -0.0021 $\pm$ 0.0005 (90\%) & -0.0021 $\pm$ 0.0005 (90\%) \\
    wine\_quality\_white & 3000 & 0 & sliced\_wasserstein & 10 & +0.0004 $\pm$ 0.0013 (50\%) & +0.0003 $\pm$ 0.0007 (60\%) & +0.0003 $\pm$ 0.0007 (60\%) \\
    wine\_quality\_white & 3000 & 0.5 & adaptive\_mmd & 10 & +0.0047 $\pm$ 0.0005 (0\%) & +0.0024 $\pm$ 0.0003 (0\%) & +0.0024 $\pm$ 0.0003 (0\%) \\
    wine\_quality\_white & 3000 & 0.5 & cart & 10 & -0.0026 $\pm$ 0.0005 (100\%) & -0.0011 $\pm$ 0.0002 (90\%) & -0.0011 $\pm$ 0.0002 (90\%) \\
    wine\_quality\_white & 3000 & 0.5 & sliced\_wasserstein & 10 & +0.0017 $\pm$ 0.0006 (20\%) & +0.0012 $\pm$ 0.0004 (20\%) & +0.0012 $\pm$ 0.0004 (20\%) \\
    \bottomrule
  \end{tabular}
  \end{adjustbox}
\end{table}

\paragraph{Multivariate responses.}
The regime that most naturally motivates distributional forests is multivariate
response, and here distributional splitting clearly earns its keep
(Table~\ref{tab:multivariate}).  On the building-energy real benchmark
(\texttt{enb}, $d=2$) and on two synthetic multivariate mechanisms, the MMD
family beats CART on the energy score with high win rates.  The cleanest case is
\texttt{paper\_copula}, a pure-dependence shift with little marginal mean signal:
every distributional criterion improves the energy score over CART by a wide
margin at a 100\% win rate, while CRPS and RMSE---which see the marginals
only---barely move.  This is exactly the structure CART cannot see and a
distributional criterion can.

\begin{table}[h]
  \centering
  \caption{Paired differences (criterion $-$ \texttt{cart}). Negative means the row improves over the reference; values are mean $\pm$ paired SE with seed-level win rate.}
  \label{tab:multivariate}
  \begin{adjustbox}{max width=\textwidth}
  \begin{tabular}{lccccccc}
    \toprule
    dataset & $n_{\mathrm{tr}}$ & honesty & criterion & pairs & $\Delta$energy & $\Delta$CRPS & $\Delta$RMSE \\
    \midrule
    enb & 576 & 0 & adaptive\_mmd & 5 & -0.0047 $\pm$ 0.0026 (100\%) & -0.0033 $\pm$ 0.0018 (100\%) & -0.0062 $\pm$ 0.0038 (80\%) \\
    enb & 576 & 0 & anisotropic\_mmd & 5 & -0.0037 $\pm$ 0.0024 (80\%) & -0.0026 $\pm$ 0.0017 (80\%) & -0.0048 $\pm$ 0.0041 (80\%) \\
    enb & 576 & 0 & mmd\_rff & 5 & -0.0024 $\pm$ 0.0017 (80\%) & -0.0017 $\pm$ 0.0011 (80\%) & -0.0027 $\pm$ 0.0027 (60\%) \\
    enb & 576 & 0 & sliced\_wasserstein & 5 & +0.0015 $\pm$ 0.0022 (40\%) & +0.0012 $\pm$ 0.0016 (40\%) & +0.0063 $\pm$ 0.0040 (20\%) \\
    enb & 576 & 0.5 & adaptive\_mmd & 5 & -0.0072 $\pm$ 0.0031 (100\%) & -0.0052 $\pm$ 0.0022 (100\%) & -0.0135 $\pm$ 0.0067 (100\%) \\
    enb & 576 & 0.5 & anisotropic\_mmd & 5 & -0.0041 $\pm$ 0.0019 (80\%) & -0.0028 $\pm$ 0.0014 (80\%) & -0.0078 $\pm$ 0.0047 (80\%) \\
    enb & 576 & 0.5 & mmd\_rff & 5 & -0.0036 $\pm$ 0.0017 (100\%) & -0.0027 $\pm$ 0.0013 (100\%) & -0.0071 $\pm$ 0.0037 (80\%) \\
    enb & 576 & 0.5 & sliced\_wasserstein & 5 & +0.0022 $\pm$ 0.0027 (40\%) & +0.0020 $\pm$ 0.0020 (20\%) & +0.0018 $\pm$ 0.0058 (60\%) \\
    paper\_copula & 3750 & 0 & adaptive\_mmd & 5 & -0.0054 $\pm$ 0.0010 (100\%) & -0.0007 $\pm$ 0.0004 (80\%) & -0.0016 $\pm$ 0.0008 (100\%) \\
    paper\_copula & 3750 & 0 & anisotropic\_mmd & 5 & -0.0050 $\pm$ 0.0014 (100\%) & -0.0002 $\pm$ 0.0006 (60\%) & -0.0018 $\pm$ 0.0011 (60\%) \\
    paper\_copula & 3750 & 0 & mmd\_rff & 5 & -0.0045 $\pm$ 0.0007 (100\%) & -0.0004 $\pm$ 0.0003 (60\%) & -0.0009 $\pm$ 0.0005 (60\%) \\
    paper\_copula & 3750 & 0 & sliced\_wasserstein & 5 & -0.0047 $\pm$ 0.0010 (100\%) & -0.0002 $\pm$ 0.0004 (80\%) & -0.0006 $\pm$ 0.0006 (80\%) \\
    paper\_copula & 3750 & 0.5 & adaptive\_mmd & 5 & -0.0040 $\pm$ 0.0004 (100\%) & +0.0001 $\pm$ 0.0001 (20\%) & +0.0001 $\pm$ 0.0003 (60\%) \\
    paper\_copula & 3750 & 0.5 & anisotropic\_mmd & 5 & -0.0043 $\pm$ 0.0002 (100\%) & +0.0004 $\pm$ 0.0002 (20\%) & +0.0004 $\pm$ 0.0003 (20\%) \\
    paper\_copula & 3750 & 0.5 & mmd\_rff & 5 & -0.0038 $\pm$ 0.0003 (100\%) & +0.0001 $\pm$ 0.0001 (40\%) & +0.0002 $\pm$ 0.0003 (40\%) \\
    paper\_copula & 3750 & 0.5 & sliced\_wasserstein & 5 & -0.0046 $\pm$ 0.0002 (100\%) & +0.0001 $\pm$ 0.0001 (20\%) & +0.0000 $\pm$ 0.0001 (60\%) \\
    paper\_heterogeneous\_regression & 3750 & 0 & adaptive\_mmd & 5 & -0.0009 $\pm$ 0.0012 (80\%) & -0.0006 $\pm$ 0.0008 (80\%) & +0.0004 $\pm$ 0.0022 (40\%) \\
    paper\_heterogeneous\_regression & 3750 & 0 & anisotropic\_mmd & 5 & -0.0002 $\pm$ 0.0010 (40\%) & -0.0001 $\pm$ 0.0006 (40\%) & +0.0008 $\pm$ 0.0020 (40\%) \\
    paper\_heterogeneous\_regression & 3750 & 0 & mmd\_rff & 5 & -0.0010 $\pm$ 0.0012 (80\%) & -0.0006 $\pm$ 0.0008 (80\%) & +0.0012 $\pm$ 0.0018 (40\%) \\
    paper\_heterogeneous\_regression & 3750 & 0 & sliced\_wasserstein & 5 & -0.0017 $\pm$ 0.0008 (80\%) & -0.0011 $\pm$ 0.0005 (80\%) & -0.0000 $\pm$ 0.0015 (40\%) \\
    paper\_heterogeneous\_regression & 3750 & 0.5 & adaptive\_mmd & 5 & -0.0062 $\pm$ 0.0012 (100\%) & -0.0039 $\pm$ 0.0007 (100\%) & -0.0003 $\pm$ 0.0018 (20\%) \\
    paper\_heterogeneous\_regression & 3750 & 0.5 & anisotropic\_mmd & 5 & -0.0067 $\pm$ 0.0009 (100\%) & -0.0042 $\pm$ 0.0005 (100\%) & -0.0016 $\pm$ 0.0014 (80\%) \\
    paper\_heterogeneous\_regression & 3750 & 0.5 & mmd\_rff & 5 & -0.0069 $\pm$ 0.0008 (100\%) & -0.0042 $\pm$ 0.0004 (100\%) & -0.0023 $\pm$ 0.0010 (100\%) \\
    paper\_heterogeneous\_regression & 3750 & 0.5 & sliced\_wasserstein & 5 & -0.0050 $\pm$ 0.0007 (100\%) & -0.0031 $\pm$ 0.0004 (100\%) & -0.0010 $\pm$ 0.0007 (80\%) \\
    \bottomrule
  \end{tabular}
  \end{adjustbox}
\end{table}

\section{Characterization}

The empirical pattern can be summarized as follows.

\begin{quote}
Distributional splitting helps when non-location conditional structure is present
and estimable at node scale.  For scalar responses in mean-dominated finite
samples, CART-style splitting is the better default; among distributional
criteria, ordinary isotropic MMD is already close to best in class, and the
anisotropic, adaptive, and sliced-Wasserstein refinements do not systematically
improve on it.  For multivariate responses, where dependence structure is
invisible to mean splitting, distributional criteria are the right tool.
\end{quote}

This statement is intentionally scoped.  It is not a claim that distributional
forests are inferior to quantile regression forests, nor that MMD is a poor
criterion in general.  The core motivation for DRF is multivariate response
distributions, where CART/QRF-style univariate orderings are not a complete
solution, and the multivariate evidence above supports that motivation directly.
The negative result is about the common univariate tabular setting: when the
target is scalar and the conditional law is mostly a shifted version of itself, a
distributional split criterion is often paying variance to chase weak structure.

\section{Implications}

For practice, the recommendation is simple.  Use CART-style splitting as the
univariate default, then reach for a distributional criterion---ordinary
isotropic MMD, without further refinement---when the response is multivariate or
the scalar data show clear heteroskedasticity, tail behavior, or non-location
structure.  When distributional splitting is used, paired-seed comparisons and
honesty settings should be reported: in these experiments, honesty can be as
important as the criterion itself.  All of the above is implemented in the
open-source \texttt{drforest} library, which provides the honest forest, every
criterion compared here, and the post-hoc shrinkage variants behind one
interface; its Rust-backed split search ran the full paired sweep in this note
at a cost that makes such characterizations routine rather than a project in
themselves.

For research, the next useful theory is not a generic dominance theorem.  The
useful object is a finite-sample comparison of split-score signal-to-noise: mean
separation versus kernel mean separation under local alternatives.
Appendix~\ref{app:finite} carries this out for a single fixed split and shows
that the location-shift mechanism is visible to both criteria, whereas the
mean-preserving scale, shape, and dependence mechanisms have zero population
mean contrast but positive Gaussian-kernel contrast.  This matches the
qualitative visibility pattern in the experiments.  What remains genuinely hard
is lifting this fixed-split calculation to the adaptive maximization over
correlated cutpoints, the
per-node random-feature resampling, and the forest aggregation that ultimately
determines CRPS or energy risk; that is where a sharp threshold statement, as
opposed to a mechanism, would have to come from.

\section{Limitations}

The scalar experiments compare CART-style splitting to MMD splitting inside the
same forest implementation; a full external QRF package comparison is a separate
baseline.  The multivariate evidence, while decisive on the energy score, uses a
small number of mechanisms and one real benchmark, and a broader multivariate
suite would strengthen the positive side of the characterization.

\bibliographystyle{plainnat}
\bibliography{main}

@article{meinshausen2006,
  author = {Meinshausen, Nicolai},
  title = {Quantile Regression Forests},
  journal = {Journal of Machine Learning Research},
  year = {2006},
  volume = {7},
  pages = {983--999}
}

@article{gretton2012,
  author = {Gretton, Arthur and Borgwardt, Karsten M. and Rasch, Malte J. and
            Sch{\"o}lkopf, Bernhard and Smola, Alexander},
  title = {A Kernel Two-Sample Test},
  journal = {Journal of Machine Learning Research},
  year = {2012},
  volume = {13},
  pages = {723--773}
}

@inproceedings{rahimi2007,
  author = {Rahimi, Ali and Recht, Benjamin},
  title = {Random Features for Large-Scale Kernel Machines},
  booktitle = {Advances in Neural Information Processing Systems},
  year = {2007},
  volume = {20}
}

@article{choi2024,
  author = {Choi, Ikjun and Kim, Ilmun},
  title = {Computational-Statistical Trade-off in Kernel Two-Sample Testing
           with Random Fourier Features},
  journal = {arXiv preprint arXiv:2407.08976},
  year = {2024}
}

@article{cevid2022,
  author = {{\v C}evid, Domagoj and Michel, Loris and N{\"a}f, Jeffrey and
            B{\"u}hlmann, Peter and Meinshausen, Nicolai},
  title = {Distributional Random Forests: Heterogeneity Adjustment and
           Multivariate Distributional Regression},
  journal = {Journal of Machine Learning Research},
  year = {2022},
  volume = {23},
  number = {333},
  pages = {1--79}
}

\appendix
\section{Finite-Node Analysis of Split Scores}
\label{app:finite}

This appendix collects the tractable theory behind the signal-allocation
characterization.  Throughout we fix a single node and a single candidate split
$L \mid R$, with $n$ observations, child sizes $n_L,n_R$, and balance
$q = n_L/n$.  Conditional on the child memberships we treat the within-child
responses as i.i.d.\ draws from the two child laws; this is a fixed-split
statement, not a statement about the cutpoint actually selected by maximizing the
score, and the gap between the two is part of what makes a sharp threshold
result hard (Remark~\ref{rem:hard}).

For the RFF criterion, all statements below condition on the sampled frequencies
and on the forest-level bandwidth.  We identify the complex feature vector in
$\mathbb C^B$ with its real representation in $\mathbb R^{2B}$ when invoking
covariances and Gaussian limits.

Let $g(\cdot)$ denote the response embedding scored by the criterion: $g(Y)=Y$
for CART, and $g(Y)=B^{-1/2}\psi_B(Y)$ for the normalized $B$-feature RFF map, so
that $\|g\|$ is bounded and the MMD score below is the implemented one.  Write
$\mu_{g,L},\mu_{g,R}$ for the child means of $g(Y)$, $C_L,C_R$ for its child
covariance operators, and
\[
  \delta_g = \mu_{g,L}-\mu_{g,R},
  \qquad
  G_g = q(1-q)\,\|\delta_g\|^2 .
\]
Both criteria score the same biased squared mean-embedding statistic, differing
only through $g$:
\[
  \widehat G_g = q(1-q)\,\bigl\|\bar g_L - \bar g_R\bigr\|^2,
  \qquad
  \bar g_L = \tfrac{1}{n_L}\!\sum_{i\in L} g(Y_i),\quad
  \bar g_R = \tfrac{1}{n_R}\!\sum_{i\in R} g(Y_i).
\]

\subsection{Exact finite-sample expectation}

\begin{proposition}\label{prop:exp}
Conditional on the child sizes,
\[
  \mathbb E\bigl[\widehat G_g\bigr]
  = G_g
  + \frac{(1-q)\operatorname{tr}(C_L) + q\operatorname{tr}(C_R)}{n} .
\]
\end{proposition}

\begin{proof}
The increment $\bar g_L-\bar g_R$ has mean $\delta_g$ and covariance
$C_L/n_L + C_R/n_R$, so
$\mathbb E\|\bar g_L-\bar g_R\|^2 = \|\delta_g\|^2 + \operatorname{tr}(C_L)/n_L +
\operatorname{tr}(C_R)/n_R$.  Multiplying by $q(1-q)$ and using
$q(1-q)/n_L=(1-q)/n$ and $q(1-q)/n_R=q/n$ gives the claim.
\end{proof}

Three consequences are worth stating. The population signal is $\|\delta_g\|^2$;
the score carries an $O(n^{-1})$ noise floor; and when the two children share a
law ($C_L=C_R=C$) that floor is $\operatorname{tr}(C)/n$ regardless of split
balance.  CART and MMD obey the same decomposition---only the embedding $g$, and
hence $\operatorname{tr}(C_\bullet)$, changes.  This decomposition does not by
itself explain the honesty comparison: honesty reduces the structure-sample
size used to score splits while separating structure selection from leaf
estimation.

\subsection{Local quadratic-form limit and signal-to-noise}

Assume finite second moments, $q_n\to q\in(0,1)$, and convergence of the child
covariance operators.  Let $W_g = C_L/q + C_R/(1-q)$, the limiting $n$-scaled
covariance of the increment.  A multivariate central limit theorem gives
\[
  \sqrt n\,\bigl((\bar g_L-\bar g_R) - \delta_g\bigr)
  \;\Rightarrow\; \mathcal N(0, W_g).
\]
Under the local alternative $\delta_{g,n}=h_g/\sqrt n$,
\[
  n\,\widehat G_g \;\Rightarrow\; q(1-q)\,\|Z_g + h_g\|^2,
  \qquad Z_g \sim \mathcal N(0,W_g) .
\]
Setting $h_g=0$ recovers the null: for CART ($g(Y)=Y$, scalar) this is a scaled
$\chi^2_1$, whereas for the finite-$B$ RFF embedding it is a weighted central
quadratic form in Gaussians.  Under $h_g\ne0$ the corresponding local
limit is generally noncentral.  Under the Gaussian approximation a natural
fixed-split signal-to-noise
ratio follows from $\mathbb E\|Z\|^2=\|\delta\|^2+\operatorname{tr}\Sigma$ and
$\operatorname{Var}\|Z\|^2 = 2\operatorname{tr}(\Sigma^2)+4\delta^\top\Sigma\delta$
with $\Sigma=W_g/n$:
\[
  \operatorname{SNR}_g
  \;\approx\;
  \frac{\|\delta_g\|^2}
       {\sqrt{\,4\langle \delta_g, W_g\,\delta_g\rangle/n
              + 2\|W_g\|_{\mathrm{HS}}^2/n^2\,}} .
\]
This formalizes signal allocation correctly: power depends on the embedding
contrast $\|\delta_g\|^2$ relative to the covariance geometry $W_g$ of the
embedded response, not on a bare count of feature coordinates.  A normalized
bounded kernel does not incur an automatic dimension penalty.

\subsection{Two-regime visibility and the synthetic mechanisms}

\begin{proposition}\label{prop:visibility}
Suppose that, within the node, each side is a mixture of the same two fixed laws
$P_0,P_1$ with regime proportions $\alpha_L,\alpha_R$.  Writing
$m_{g,i}=\mathbb E_{P_i}[g]$,
\[
  \delta_g = (\alpha_L-\alpha_R)\,(m_{g,1}-m_{g,0}),
  \qquad
  G_g = q(1-q)\,(\alpha_L-\alpha_R)^2\,\|m_{g,1}-m_{g,0}\|^2 .
\]
\end{proposition}

The cutpoint-dependent factor $q(1-q)(\alpha_L-\alpha_R)^2$ is shared by every
embedding; criteria differ only through the constant $\|m_{g,1}-m_{g,0}\|^2$.
Hence every embedding that can see $P_0\neq P_1$ induces the \emph{same}
population ranking over cutpoints, up to that constant.  CART's constant is
$\|\mathbb E_{P_1}Y-\mathbb E_{P_0}Y\|^2$ and vanishes exactly when the two
regimes share a mean.  On \texttt{paper\_quantile\_1}, a pure
location shift, CART and MMD share population signal and any difference is
finite-sample.  On \texttt{paper\_quantile\_2} (equal means, scale
$1\!\to\!2$), \texttt{paper\_quantile\_3} (Gaussian vs.\ exponential, both unit
mean), CART's population contrast is exactly zero while the Gaussian-kernel
contrast is positive.

The copula mechanism is not a two-regime model: its conditional law varies
continuously with $X_1\sim\operatorname{Unif}(0,1)$.  Nevertheless, any interior
cut $X_1=t$ has zero child-mean contrast, while for distinct response coordinates
$j\ne k$ the child covariances are
$\operatorname{Cov}(Y_j,Y_k\mid X_1\le t)=t/2$ and
$\operatorname{Cov}(Y_j,Y_k\mid X_1>t)=(1+t)/2$.  Thus the joint child laws
differ, and their exact Gaussian-kernel MMD is positive.

\begin{table}[h]
  \centering
  \caption{Population split visibility for the synthetic diagnostics.  The
  quantile mechanisms change regime at $X_1=0$; the copula result holds at any
  interior cut $X_1=t\in(0,1)$.  ``Mean'' is the CART contrast
  $\|\delta_Y\|$ and ``kernel'' is the exact Gaussian-MMD contrast.}
  \label{tab:signal}
  \begin{tabular}{lccc}
    \toprule
    mechanism & distributional change & mean signal & kernel signal \\
    \midrule
    \texttt{paper\_quantile\_1} & $\mathcal N(0.8,1)$ vs.\ $\mathcal N(0,1)$ & yes & yes \\
    \texttt{paper\_quantile\_2} & $\mathcal N(0,4)$ vs.\ $\mathcal N(0,1)$ & $0$ & yes \\
    \texttt{paper\_quantile\_3} & $\mathrm{Exp}(1)$ vs.\ $\mathcal N(1,1)$ & $0$ & yes \\
    \texttt{paper\_copula} & corr.\ $X_1$, fixed $\mathcal N(0,1)$ marginals & $0$ & joint only \\
    \bottomrule
  \end{tabular}
\end{table}

For the Gaussian kernel $k(y,y')=\exp\{-(y-y')^2/2\sigma^2\}$ the first two
scalar contrasts are closed-form, using
$\mathbb E_{\mathcal N(m,v)}\exp\{-U^2/2\sigma^2\}=
\frac{\sigma}{\sqrt{\sigma^2+v}}\exp\{-m^2/2(\sigma^2+v)\}$:
\[
  \mmd^2\bigl\{\mathcal N(0,1),\mathcal N(0.8,1)\bigr\}
  = \frac{2\sigma}{\sqrt{\sigma^2+2}}
    \left[1-\exp\!\left\{-\frac{0.8^2}{2(\sigma^2+2)}\right\}\right],
\]
\[
  \mmd^2\bigl\{\mathcal N(0,1),\mathcal N(0,4)\bigr\}
  = \frac{\sigma}{\sqrt{\sigma^2+2}}
  + \frac{\sigma}{\sqrt{\sigma^2+8}}
  - \frac{2\sigma}{\sqrt{\sigma^2+5}} ,
\]
which are exactly the \texttt{paper\_quantile\_1} and \texttt{paper\_quantile\_2}
splits.  Both are strictly positive for every bandwidth $\sigma$: the first
confirms that both criteria see the location shift, while the second retains
kernel signal despite a zero CART contrast.

\subsection{What a uniform bound can and cannot give}

Fix the bandwidth and condition on the predictor-derived candidate grid.  For a
finite candidate set $\mathcal S$, let $G_k(s)$ be the exact-kernel population
score and $\widehat G_B(s)$ the empirical RFF score.  Bounded-feature
concentration and a union bound over $\mathcal S$ yield a uniform control of
schematic form
\[
  \sup_{s\in\mathcal S}\bigl|\widehat G_B(s)-G_k(s)\bigr|
  \;\lesssim\;
  \sqrt{\tfrac{\log(|\mathcal S|/\alpha)}{n_{\min}}}
  + \tfrac{\log(|\mathcal S|/\alpha)}{n_{\min}}
  + \sqrt{\tfrac{\log(|\mathcal S|/\alpha)}{B}},
\]
whose three terms are within-node response sampling, the quadratic remainder, and
the random-feature approximation \citep{rahimi2007,gretton2012}.  On an event
where the right-hand side is at most $\varepsilon$, the empirical maximizer has
population regret at most $2\varepsilon$; if the unique population maximizer is
separated from every competitor by more than $2\varepsilon$, it is recovered
exactly.  The same argument gives informative-feature recovery when that margin
separates the best informative split from all uninformative splits.  The bound
captures minimum child size $n_{\min}$, the number of searched cutpoints
$|\mathcal S|$, and the feature budget $B$.  We state it schematically because a
sharp constant is not the point and is genuinely harder than the bound suggests
(Remark~\ref{rem:hard}).

\begin{remark}\label{rem:hard}
A universal threshold at which MMD splitting overtakes CART is not available
without heavily restricting the model.  There is no universal ordering: under a
fixed-split homoskedastic Gaussian location alternative, the standardized mean
difference is the likelihood-ratio statistic and CART is locally optimal.
Under mean-preserving variance, shape, or dependence alternatives, CART has zero
population contrast while an exact characteristic-kernel MMD has positive
contrast; a fixed finite RFF map additionally requires that its sampled
frequencies see the alternative.  Beyond this, tree building
maximizes the score over strongly correlated cutpoints (so fixed-split SNR is not
selection probability), the RFF frequencies are resampled per node, and the
median bandwidth is estimated once from the training responses and held fixed
for the forest.  Recursive splitting changes later node laws, and converting
root-node selection power into final CRPS or energy risk requires analyzing the
adaptive forest partition and its aggregation.  The original DRF
analysis already supplies the RKHS--CART interpretation and whole-forest
consistency \citep{cevid2022}, and fixed-$B$ RFF cannot support unrestricted
consistency claims \citep{choi2024}; the contribution here is the exact
finite-node comparison, not a dominance theorem.
\end{remark}

\section{Shrinkage Frontier}

Post-hoc shrinkage of forest weights is a separate lever from the split
criterion.  Table~\ref{tab:shrinkage} pairs each shrinkage variant against raw
weights within a fixed criterion and forest, so the difference isolates the
shrinkage effect.  Shrinkage does not change the criterion ranking: marginal
shrinkage is at best a small CRPS improvement on one synthetic mechanism, parent
shrinkage runs to large intensity and consistently hurts, and the pattern is the
same across CART, MMD, and sliced-Wasserstein forests.  Shrinkage is not the
missing ingredient that would let the MMD variants dominate CART on scalar data.

\begin{table}[h]
  \centering
  \caption{Shrinkage frontier: paired differences (variant $-$ \texttt{raw}) within a fixed criterion and forest. Negative means shrinkage improves over raw weights; $\bar\alpha$ is the mean shrinkage intensity.}
  \label{tab:shrinkage}
  \begin{adjustbox}{max width=\textwidth}
  \begin{tabular}{lcccccccc}
    \toprule
    dataset & criterion & honesty & variant & pairs & $\bar\alpha$ & $\Delta$RMSE & $\Delta$CRPS & $\Delta$energy \\
    \midrule
    enb & cart & 0.5 & marginal\_kmse & 5 & 0.003 & -0.0001 $\pm$ 0.0008 (40\%) & +0.0009 $\pm$ 0.0001 (0\%) & +0.0012 $\pm$ 0.0002 (0\%) \\
    enb & cart & 0.5 & marginal\_stein & 5 & 0.003 & -0.0001 $\pm$ 0.0008 (40\%) & +0.0009 $\pm$ 0.0001 (0\%) & +0.0012 $\pm$ 0.0002 (0\%) \\
    enb & cart & 0.5 & parent\_kmse & 5 & 0.282 & +0.0472 $\pm$ 0.0053 (0\%) & +0.0351 $\pm$ 0.0017 (0\%) & +0.0509 $\pm$ 0.0022 (0\%) \\
    enb & cart & 0.5 & parent\_stein & 5 & 0.402 & +0.0612 $\pm$ 0.0084 (0\%) & +0.0494 $\pm$ 0.0029 (0\%) & +0.0719 $\pm$ 0.0039 (0\%) \\
    enb & mmd\_rff & 0.5 & marginal\_kmse & 5 & 0.003 & -0.0001 $\pm$ 0.0008 (40\%) & +0.0009 $\pm$ 0.0001 (0\%) & +0.0013 $\pm$ 0.0002 (0\%) \\
    enb & mmd\_rff & 0.5 & marginal\_stein & 5 & 0.003 & -0.0001 $\pm$ 0.0008 (40\%) & +0.0009 $\pm$ 0.0001 (0\%) & +0.0013 $\pm$ 0.0002 (0\%) \\
    enb & mmd\_rff & 0.5 & parent\_kmse & 5 & 0.286 & +0.0487 $\pm$ 0.0044 (0\%) & +0.0361 $\pm$ 0.0013 (0\%) & +0.0522 $\pm$ 0.0016 (0\%) \\
    enb & mmd\_rff & 0.5 & parent\_stein & 5 & 0.408 & +0.0623 $\pm$ 0.0070 (0\%) & +0.0507 $\pm$ 0.0023 (0\%) & +0.0735 $\pm$ 0.0030 (0\%) \\
    enb & sliced\_wasserstein & 0.5 & marginal\_kmse & 5 & 0.003 & -0.0001 $\pm$ 0.0008 (40\%) & +0.0009 $\pm$ 0.0001 (0\%) & +0.0012 $\pm$ 0.0002 (0\%) \\
    enb & sliced\_wasserstein & 0.5 & marginal\_stein & 5 & 0.003 & -0.0002 $\pm$ 0.0008 (40\%) & +0.0009 $\pm$ 0.0001 (0\%) & +0.0012 $\pm$ 0.0002 (0\%) \\
    enb & sliced\_wasserstein & 0.5 & parent\_kmse & 5 & 0.275 & +0.0495 $\pm$ 0.0052 (0\%) & +0.0360 $\pm$ 0.0015 (0\%) & +0.0520 $\pm$ 0.0019 (0\%) \\
    enb & sliced\_wasserstein & 0.5 & parent\_stein & 5 & 0.396 & +0.0644 $\pm$ 0.0081 (0\%) & +0.0510 $\pm$ 0.0027 (0\%) & +0.0740 $\pm$ 0.0035 (0\%) \\
    paper\_quantile\_2 & cart & 0.5 & marginal\_kmse & 5 & 0.316 & -0.0007 $\pm$ 0.0002 (100\%) & +0.0018 $\pm$ 0.0003 (0\%) & +0.0018 $\pm$ 0.0003 (0\%) \\
    paper\_quantile\_2 & cart & 0.5 & marginal\_stein & 5 & 0.460 & -0.0010 $\pm$ 0.0003 (100\%) & +0.0026 $\pm$ 0.0005 (0\%) & +0.0026 $\pm$ 0.0005 (0\%) \\
    paper\_quantile\_2 & cart & 0.5 & parent\_kmse & 5 & 0.839 & -0.0004 $\pm$ 0.0003 (80\%) & +0.0009 $\pm$ 0.0002 (0\%) & +0.0009 $\pm$ 0.0002 (0\%) \\
    paper\_quantile\_2 & cart & 0.5 & parent\_stein & 5 & 1.000 & -0.0004 $\pm$ 0.0004 (80\%) & +0.0013 $\pm$ 0.0003 (0\%) & +0.0013 $\pm$ 0.0003 (0\%) \\
    paper\_quantile\_2 & mmd\_rff & 0.5 & marginal\_kmse & 5 & 0.059 & -0.0001 $\pm$ 0.0001 (80\%) & +0.0002 $\pm$ 0.0001 (20\%) & +0.0002 $\pm$ 0.0001 (20\%) \\
    paper\_quantile\_2 & mmd\_rff & 0.5 & marginal\_stein & 5 & 0.066 & -0.0002 $\pm$ 0.0001 (80\%) & +0.0002 $\pm$ 0.0002 (40\%) & +0.0002 $\pm$ 0.0002 (40\%) \\
    paper\_quantile\_2 & mmd\_rff & 0.5 & parent\_kmse & 5 & 0.903 & -0.0002 $\pm$ 0.0005 (60\%) & -0.0001 $\pm$ 0.0003 (60\%) & -0.0001 $\pm$ 0.0003 (60\%) \\
    paper\_quantile\_2 & mmd\_rff & 0.5 & parent\_stein & 5 & 1.000 & -0.0003 $\pm$ 0.0005 (60\%) & -0.0002 $\pm$ 0.0003 (60\%) & -0.0002 $\pm$ 0.0003 (60\%) \\
    paper\_quantile\_2 & sliced\_wasserstein & 0.5 & marginal\_kmse & 5 & 0.057 & -0.0002 $\pm$ 0.0001 (80\%) & +0.0002 $\pm$ 0.0001 (20\%) & +0.0002 $\pm$ 0.0001 (20\%) \\
    paper\_quantile\_2 & sliced\_wasserstein & 0.5 & marginal\_stein & 5 & 0.064 & -0.0001 $\pm$ 0.0000 (80\%) & +0.0002 $\pm$ 0.0002 (20\%) & +0.0002 $\pm$ 0.0002 (20\%) \\
    paper\_quantile\_2 & sliced\_wasserstein & 0.5 & parent\_kmse & 5 & 0.904 & -0.0000 $\pm$ 0.0003 (60\%) & +0.0001 $\pm$ 0.0002 (40\%) & +0.0001 $\pm$ 0.0002 (40\%) \\
    paper\_quantile\_2 & sliced\_wasserstein & 0.5 & parent\_stein & 5 & 1.000 & -0.0000 $\pm$ 0.0004 (60\%) & +0.0001 $\pm$ 0.0003 (40\%) & +0.0001 $\pm$ 0.0003 (40\%) \\
    shrinkage\_toy & cart & 0.5 & marginal\_kmse & 5 & 0.040 & +0.0076 $\pm$ 0.0018 (0\%) & +0.0038 $\pm$ 0.0007 (0\%) & +0.0061 $\pm$ 0.0011 (0\%) \\
    shrinkage\_toy & cart & 0.5 & marginal\_stein & 5 & 0.042 & +0.0081 $\pm$ 0.0020 (0\%) & +0.0040 $\pm$ 0.0007 (0\%) & +0.0064 $\pm$ 0.0012 (0\%) \\
    shrinkage\_toy & cart & 0.5 & parent\_kmse & 5 & 0.379 & +0.0137 $\pm$ 0.0036 (0\%) & +0.0072 $\pm$ 0.0014 (0\%) & +0.0135 $\pm$ 0.0023 (0\%) \\
    shrinkage\_toy & cart & 0.5 & parent\_stein & 5 & 0.586 & +0.0229 $\pm$ 0.0053 (0\%) & +0.0119 $\pm$ 0.0023 (0\%) & +0.0223 $\pm$ 0.0036 (0\%) \\
    shrinkage\_toy & mmd\_rff & 0.5 & marginal\_kmse & 5 & 0.039 & +0.0075 $\pm$ 0.0019 (0\%) & +0.0037 $\pm$ 0.0007 (0\%) & +0.0060 $\pm$ 0.0011 (0\%) \\
    shrinkage\_toy & mmd\_rff & 0.5 & marginal\_stein & 5 & 0.041 & +0.0079 $\pm$ 0.0020 (0\%) & +0.0039 $\pm$ 0.0007 (0\%) & +0.0063 $\pm$ 0.0012 (0\%) \\
    shrinkage\_toy & mmd\_rff & 0.5 & parent\_kmse & 5 & 0.376 & +0.0131 $\pm$ 0.0036 (0\%) & +0.0068 $\pm$ 0.0015 (0\%) & +0.0129 $\pm$ 0.0024 (0\%) \\
    shrinkage\_toy & mmd\_rff & 0.5 & parent\_stein & 5 & 0.578 & +0.0216 $\pm$ 0.0054 (0\%) & +0.0112 $\pm$ 0.0024 (0\%) & +0.0211 $\pm$ 0.0037 (0\%) \\
    shrinkage\_toy & sliced\_wasserstein & 0.5 & marginal\_kmse & 5 & 0.039 & +0.0076 $\pm$ 0.0019 (0\%) & +0.0038 $\pm$ 0.0007 (0\%) & +0.0061 $\pm$ 0.0011 (0\%) \\
    shrinkage\_toy & sliced\_wasserstein & 0.5 & marginal\_stein & 5 & 0.042 & +0.0081 $\pm$ 0.0020 (0\%) & +0.0040 $\pm$ 0.0007 (0\%) & +0.0065 $\pm$ 0.0012 (0\%) \\
    shrinkage\_toy & sliced\_wasserstein & 0.5 & parent\_kmse & 5 & 0.383 & +0.0131 $\pm$ 0.0034 (0\%) & +0.0068 $\pm$ 0.0013 (0\%) & +0.0130 $\pm$ 0.0021 (0\%) \\
    shrinkage\_toy & sliced\_wasserstein & 0.5 & parent\_stein & 5 & 0.591 & +0.0220 $\pm$ 0.0051 (0\%) & +0.0114 $\pm$ 0.0022 (0\%) & +0.0216 $\pm$ 0.0034 (0\%) \\
    \bottomrule
  \end{tabular}
  \end{adjustbox}
\end{table}

\end{document}